\documentclass[10pt,shortpaper,twoside]{IEEEtran}

\usepackage{cite}
\usepackage{wrapfig}
\usepackage{amsthm}
\usepackage{nccmath}
\usepackage{cancel}
\usepackage{amsmath,amsfonts}
\usepackage{siunitx}
\usepackage{graphicx}
\usepackage{lipsum}  
\usepackage{textcomp}
\usepackage{balance}
\usepackage{hyperref}
\usepackage{xcolor}
\usepackage{wrapfig}

\usepackage{algorithm,algorithmic}

\usepackage[left=48pt, right=48pt, bottom=44pt, top=60pt]{geometry}
\def\BibTeX{{\rm B\kern-.05em{\sc i\kern-.025em b}\kern-.08em
		T\kern-.1667em\lower.7ex\hbox{E}\kern-.125emX}}
\usepackage{amssymb,verbatim,bm,color,times,mathtools,multicol,cuted}

\newcommand{\matrixstyle}[1]{\mathrm{#1}}
\newcommand{\vectorstyle}[1]{\boldsymbol{\mathbf{#1}}}

\newcommand{\expect}[1]{\mathbb{E}_{#1}}

\newcommand{\kullbacks}[2]{\mathbb{D}\left[#1\parallel#2\right]}

\newtheorem{theorem}{\textbf{Theorem}}

\newtheorem{lemma}{\textbf{Lemma}}
\newtheorem{corollary}{\textbf{Corollary}}

\begin{document}
%	\linenumbers
	
	\title{Entropy Regularised Deterministic Optimal Control: \\ From Path Integral Solution to Sample-Based Trajectory Optimisation}
	\author{
		Tom Lefebvre, Guillaume Crevecoeur 
		\vspace*{-18pt}
		\thanks{\{Tom.Lefebvre,Guillaume.Crevecoeur\}@ugent.be are with Ghent University, 9000 Ghent, Belgium; and with EEDT-DC, Flanders Make, Belgium.
		}
		\thanks{Manuscript received Month XX, 20XX; revised Month XX, 20XX;
			accepted Month XX, 20XX. Date of publication Month XX, 20XX;
			date of current version  Month XX, 20XX. Recommended by Senior
			Editor YY}
	}
	
	\maketitle
	
	\begin{abstract} 
		Sample-based trajectory optimisers are a promising tool for the control of robotics with non-differentiable dynamics and cost functions. Contemporary approaches derive from a restricted subclass of stochastic optimal control where the optimal policy can be expressed in terms of an expectation over stochastic paths. By estimating the expectation with Monte Carlo sampling and reinterpreting the process as exploration noise, a stochastic search algorithm is obtained tailored to (deterministic) trajectory optimisation. For the purpose of future algorithmic development, it is essential to properly understand the underlying theoretical foundations that allow for a principled derivation of such methods. In this paper we make a connection between entropy regularisation in optimisation and deterministic optimal control. We then show that the optimal policy is given by a belief function rather than a deterministic function. The policy belief is governed by a Bayesian-type update where the {likelihood} can be expressed in terms of a conditional expectation over paths induced by a prior policy. Our theoretical investigation firmly roots sample-based trajectory optimisation in the larger family of control as inference. It allows us to justify a number of heuristics that are common in the literature and motivate a number of new improvements that benefit convergence.
	\end{abstract}
	
	
	\section{Introduction}
	\IEEEPARstart{T}{rajectory} optimisation is ubiquitous in robotics. It is used to synthesise complex dynamic behaviour \cite{mordatch2012trajopt} as well as to compute real-time feedback control \cite{dantec2021icra}. Currently there are two classes of algorithms that address trajectory optimisation: (1) gradient based algorithms deriving from iLQR \cite{todorov2005ilqr} and DDP \cite{mayne1966ddp} that rely on linear-quadratic approximations about a nominal trajectory, and (2) sample-based approaches akin to Model Predictive Path Integral (MPPI) control \cite{williams2016agressive,williams2017model,williams2018information}. As opposed to the first category, the second class of methods relies on sampled trajectories to probe the local optimisation landscape and use approximate inference techniques to update the solution. {These methods ought to be less prone to failure when the models are non-differentiable or many local minima exist \cite{ha2016path}. In addition, they are highly parallelisable \cite{williams2018information}.} For this reason sample-based trajectory optimisation has gained the interest of many researchers in robotics\cite{ha2018path,kahn2021badgr,nagabandi2020deep,bhardwaj2021fast}. 
	
	MPPI derives from a restricted subclass of stochastic optimal control where inference emerges \textit{naturally}; see \cite{kappen2005linear} and sec. \ref{sec:control-as-inference-through-lsoc}. This class is known as Linearly Solvable Optimal Control (LSOC) or path integral control. Essentially, the optimal policy can be expressed as a conditional expectation of an exponential cost-to-go. Because the expectation is taken over the passive stochastic dynamics, it is possible to compute the control from passive sampled trajectories. 
	
	{A body of work has proposed adaptations to MPPI by pursuing similarities with well-known gradient-based and stochastic optimisation algorithms\cite{stulp2012path,bhardwaj2021fast,rajamaki2016sampled,ghandi2021ieeeral,lefebvre2019path}. These adjustments outperform the original algorithm but lack justification.}	
	The main objective of this article is to root the derivation of sample-based trajectory optimisers in a generic theoretical framework. In fact, MPPI belongs to a larger family of ideas collectively known as \textit{control as inference}. The idea is to reformulate the optimal control problem as a probabilistic inference problem so one can draw from the computational machinery that addresses inference. There are a number of ways to do this, and there is still ongoing research to determine how the different approaches are connected to each other \cite{rawlik2013stochastic,Watson2021cai,levine2018reinforcement}. {The main classes are path integral control, entropy regularised reinforcement learning \cite{ziebart2010modeling,rawlik2013stochastic} and message passing \cite{toussaint2009robot,watson2020stochastic}.} {Message passing algorithms reformulate stochastic optimal control as Bayesian input estimation by reinterpreting the likelihood of an observation as the desirability of a state-action pair. This problem can be solved using the Bayesian filtering and smoothing equations}. { As far as we are aware of, no sample-based algorithms are related to this idea.} Entropy regularisation in reinforcement learning addresses stochastic optimal control with additional entropic terms in the control objective. This leads to a so-called soft Bellman recursive equation, which can be solved for Linear Gaussian Quadratic regulators but not for arbitrary non-linear problems. 
	
	Regardless, entropy regularisation turns out to be a fruitful direction for our purpose. Recent research illustrates how the principle of entropic inference can be put forth as a principled motivation for entropy regularisation in deterministic optimisation \cite{lefebvre2020elsoc,luo2019minima}. In this paper we make a direct connection between entropy regularised optimisation and deterministic optimal control. {To this end, we answer two questions
		\begin{enumerate}
			\item Why does entropic regularisation work in stochastic optimisation, and what is the underlying principle?
			\item Why do heuristic adjustments of MPPI, not justified by theory, work better than the original algorithm?
	\end{enumerate}}
	
	{ The first step towards answering these questions was taken in \cite{lefebvre2020elsoc}, where the underlying problem statement was still closely related to LSOC. As a result, the dynamics had to be invertible and the optimisation variable was a state transition distribution rather than a policy. To extract a locally linear feedback policy, the dynamics also needed to be control affine.}
	
	In this paper, we essentially show that we can drop these assumptions. {We specifically demonstrate that entropic deterministic control gives rise to a path integral expression for optimal control and how a method akin to MPPI can be derived from this result.} Therewith, we establish a theoretical foundation that allows the derivation of generic sample-based search algorithms tailored to deterministic trajectory optimisation.
	
	\subsection{Notation} Discrete time is denoted with subscript $n$. The iteration number in algorithms is denoted with subscript $g$. Entities in sample sets are labelled with subscript $j$. We use bold font to denote sequences, e.g. $\vectorstyle{a} = \{a_0,a_1,\dots,a_{N-1}\}$.  We use time subscript $n$ to denote a subsequence starting from time instant $n$, e.g. $\boldsymbol{\tau}_n = \{\tau_n,\tau_{n+1},\dots,\tau_N\}$. A sequence can also be expressed as $\vectorstyle{a} = \{a_n\}_n$.  We rely on context to imply the range. We combine both notations so that e.g. $\{\vectorstyle{\pi}_g\}_g$  refers to iterates of a sequence $\boldsymbol{\pi}_{g} = \{\pi_{g,n}\}_n$. Expression $x\sim p(x)$ implies that the uncertain variable $x$ is distributed according to distribution $p(x)$, $\mathbb{E}_p[f(x)]$ denotes the expected value of $f(x)$. $\mathcal{N}(x|\mu,\Sigma)$ denotes the multivariate Gaussian or Normal distribution with mean $\mu\in\mathbb{R}^n$ and positive definite covariance matrix $\Sigma\in\mathbb{S}^n_+$. $\mathcal{N}(x|\mu,\Sigma)$ is often shortened to $\mathcal{N}(\mu,\Sigma)$.

	\subsection{Problem description}\label{sec:problem-description}We consider finite horizon discrete time deterministic optimal control. Variables $s\in\mathcal{S}\subset\mathbb{R}^{n_s}$ and $a\in\mathcal{A}\subset\mathbb	{R}^{n_a}$ denote states and controls, respectively. We assume dynamics are governed by a non-linear time variant difference equation given by $s_{n+1}= f_n(s_n,a_n)$. State-action couples are denoted as $\tau_n = \{s_n,a_n\}$, except for $\tau_N = s_N$. A trajectory is defined as $\boldsymbol{\tau} = \{\mathbf{s},\mathbf{a}\}$. A feasible trajectory satisfies the dynamic difference equation. Functions $r_n:\mathcal{S}\times\mathcal{A}\rightarrow\mathbb{R}_{\geq 0}$ and $r_N:\mathcal{S}\rightarrow\mathbb{R}_{\geq 0}$ denote the running and terminal costs, respectively. The cost-to-go from $s$ is defined as
		\begin{equation*}
		R_n(\boldsymbol{\tau}_n) =  r_N(s_N)+\sum\nolimits_{n'=n}^{N-1} r_{n'}(s_{n'},a_{n'}) 
		\end{equation*}
	
	We consider the trajectory optimisation problem defined below. The problem solves for open-loop control sequence $\vectorstyle{a}^*$.
		\begin{equation}
		\label{eq:dtSOC}
		\begin{aligned}
		\min_{\mathbf{a}} R_0(\boldsymbol{\tau}_0) 
		\text{ s.t. } s_{n+1}= f_n(s_n,a_n), ~ s_0 = s_0^*
		\end{aligned}
		\end{equation}
	
	The optimal cost-to-go or value function $V_{n}(\cdot)$ is defined as $V_n(s) = \min_{\vectorstyle{a}_n} R_n(\vectorstyle{\tau}_n)$, subject to the dynamics and initial state $s$. Relying on Bellman's principle of optimality, (\ref{eq:dtSOC}) can be cast into a recursive problem. This gives rise to Bellman's backward recursion equation with boundary condition $V_N(s) = r_N(s)$. Sequence $\mathbf{a}=\{a_n(s_n)\}_n$ contains $N-1$ state-dependent closed-loop policy functions $a_n:\mathcal{S}\rightarrow\mathcal{A}$. 
		\begin{equation}
		\label{def:SBE}
		\begin{aligned}
		V_n(s) &= \min_{a} r_n(s,a) + V_{n+1}(f_n(s,a))\\
		a_n(s) &= \arg\min_{a} r_n(s,a) + V_{n+1}(f_n(s,a))
		\end{aligned}
		\end{equation}
	
	\section{Control as inference through LSOC}\label{sec:control-as-inference-through-lsoc}
	
	LSOC refers to an interesting but restrictive subclass of continuous-time stochastic optimal control problems. Here, we give a brief overview of the main ideas \cite{kappen2005linear}. 
	We consider control affine stochastic dynamics where $\text{d}\xi$ denotes a Wiener process, so that $\mathbb{E}[\text{d}\xi\text{d}\xi^\top]= \Sigma(t,s)$
		\begin{equation*}
		\text{d}s = f(t,s)\text{d}t + \matrixstyle{B}(t,s)( a \text{d}t + \text{d}\xi)  
		\end{equation*}
	A state dependent cost-to-go is defined as
		\begin{equation*}
		C(t)  = c_T(s(T)) + \int\nolimits_{t}^{T} c(\tau,s(\tau)) \text{d}\tau 
		\end{equation*}
	We look for a continuous time optimal control $a^*(t\rightarrow T)$. The cost $C(t)$ is appended with an input-dependent cost. Note that the rate is inversely proportional to the noise covariance. $V(t,s)$ defines the continuous-time value function. The expectation is taken over the path probability $P(a|s(t))$ induced by the stochastic dynamics $a(t\rightarrow T)$ and conditioning on $s(t)$.
		\begin{equation}
		\label{eq=stoc}
		V(t,s(t)) = \min_{a(t\rightarrow T)} \mathbb{E}_{P(a|s(t))}\left[C(t) + \int_t^T \tfrac{\lambda}{2} \|a(\tau)\|^2_{ \Sigma^{-1} }\text{d}\tau\right]
		\end{equation}
	
	Defining the \textit{desirability function} $Z(t,s) = - \lambda \log V(t,s)$, it can be shown that the solution of (\ref{eq=stoc}) is governed by a linear partial differential equation. Remarkably, the solution can then be expressed as a path integral according to the Feynman–Kac formula. Note that the expectation is taken over \textit{passive} paths. 
		\begin{equation*}
		Z(t,s(t)) = \mathbb{E}_{P(0|s(t)}\left[\exp\left(-\tfrac{1}{\lambda }C(t)\right)\right]
		\end{equation*}
	Second it can also be shown that the optimal control satisfies
		\begin{equation*}
		a^*(t,s(t)) = \frac{1}{Z(t,s(t))}  \mathbb{E}_{P(0|s(t)}\left[\exp\left(-\tfrac{1}{\lambda }C(t)\right) \text{d}\xi\right]
		\end{equation*}
	
	On account of Girsanov's theorem, the measure of the expectation can be changed to the system dynamics induced by any arbitrary control $a_g(t\rightarrow T)$ \cite{ha2016path}. This basically amounts to importance sampling in continuous time. Here, $\tfrac{\text{d}P_0}{\text{d}P_{a_g}}$ represents the Radon-Nikodym derivative of $P_0$ with respect to $P_{a_g}$.
		\begin{equation*}
		a^*(t,s(t)) \propto  \mathbb{E}_{P(a_g|s(t)}\left[\tfrac{\text{d}P(0)}{\text{d}P(a_g)}\exp\left(-\tfrac{1}{\lambda }C(t)\right) \text{d}\xi\right]
		\end{equation*}
	This summarises the main concepts from LSOC.
	
	From this theory, one can derive a sample-based trajectory optimisation algorithm known as MPPI. Because this (and related) method(s) rely on the calculation of a path integral, they are also referred to as path integral control. The MPPI algorithm is summarised in Alg. \ref{alg:1}. Due to Girsanov, sampled trajectories are obtained about a reference trajectory induced by $\vectorstyle{a}_g$ with simulated control perturbations $\xi_{n,j}$. An updated control $\vectorstyle{a}_{g+1}$ is inferred according to the theory of LSOC.

	\setlength{\textfloatsep}{5pt}
	\begin{algorithm}[t!]
		\algsetup{linenosize=\small}
		\small
		\caption{MPPI}
		\begin{algorithmic}[1]
			\label{algo:MPPI}
			\STATE \textbf{input }$\delta t,\{a_{0,n}\}_n,\Sigma$
			\STATE \textbf{output }$\{a_{g,n}\}_n$
			\FOR {$g=0,1,2,\dots$}
			\FOR {$j=1,2,3,\dots,M$}
			\FOR {$n=0,1,2,\dots,N-1$}
			\STATE  ${\xi}_{n,j}\sim \mathcal{N}\left(0,{\Sigma}\right)$ 
			\STATE  ${s}_{n+1,j} \leftarrow {f}_n\left({s}_{n,j}\right)\delta t + \matrixstyle{B}_n(s_{n,j}) ({a}_{g,n}\delta t + \xi_{n,j}\sqrt{\delta t}) $
			\STATE  $r_{n,j} \leftarrow c_{n}(s_{n,j})\delta t + \tfrac{\lambda}{2}  a_{g,n}^\top \Sigma^{-1} ({a}_{g,n} \delta t + 2\xi_{n,j}\sqrt{\delta t})$
			\ENDFOR
			\STATE  $r_{N,j} \leftarrow c_N(s_{N,j})\delta t$
			\STATE 	$w_{n,j} \leftarrow \exp \left(-\tfrac{1}{\lambda}\sum_n r_{n,j} \right)$ 
			\ENDFOR
			\FOR {$n=0,1,2,\dots,N-1$}
			\STATE $a_{g+1,n} = a_{g,n} + \sum_{j} \frac{w_{n,j}}{\sum_j w_{n,j}} \xi_{n,j}$
			\ENDFOR
			\ENDFOR
		\end{algorithmic}
		\label{alg:1}
	\end{algorithm}
	
	A few other remarks are in place:
	\begin{enumerate}
		\item The algorithm solves \textit{stochastic} optimal control problem (\ref{eq=stoc}). The policy itself is (thus) deterministic and inference is established by the inherent process noise.
		\item The control penalty in (\ref{eq=stoc}) is quadratic, i.e.,  $\tfrac{1}{2}\|a\|_\matrixstyle{R}$, with $\matrixstyle{R} = \lambda \Sigma^{-1}$. {  Hence, when choosing two out of the three parameters $\matrixstyle{R}$, $\Sigma$ and $\lambda$, we fix the problem that we solve.}
		\item In practice, the covariance is often updated comparable to CMA-ES \cite{stulp2012path,bhardwaj2021fast}, $\Sigma_{g+1,n} =\sum_{j} \frac{w_{n,j}}{\sum_j w_{n,j}} \xi_{n,j}\xi_{n,j}^\top$. 
		\item Also, the theoretical running cost in line 8 is replaced by a generalised running cost, $r_{n,j} \leftarrow r_n(s_{n,j},a_{g,n}+\xi_{n,j})$.
	\end{enumerate}
	
	{Although the two adjustments proposed above lead to better performance, they are not supported by any theory. }Furthermore, the combination of remarks 2 and 3 implies that we are solving different problems with every update. By consequence, we argue they are also not properly understood.
	
	\section{Entropic optimisation}\label{sec:entropic-optimisation} 
	{ In this section we answer question 1) mentioned in the introduction. It serves as a stepping stone to question 2). }
	
	Algorithms for numerical problems, such as optimisation, proceed iteratively, with each iteration providing information that improves a running estimate of the correct solution. Probabilistic numerics \cite{oates2019probnum} pursues methods that, in place of such estimates, update beliefs\footnote{The manifestation and interpretation of probability is epistemic and arises from missing information in a computation that is otherwise deterministic.} over the solution space. In brief, in this section, we aim to rephrase the generic problem of optimisation as a problem of inference. By embedding deterministic optimisation into the framework of entropic inference, we provide a theoretical argument for the use of entropy regularisation in optimisation in addition to the vast empirical validation documented in previous work.

	\subsection{Entropic inference}
	In probability theory, inference refers to the rational processing of incomplete information \cite{jaynes2003}. In the present context, we make use of probabilities to encode our uncertainty about an underlying deterministic quantity and refer to them as beliefs. We seek a posterior, $\pi$, which  encodes new information into a prior, $\rho$, which encodes information that we already have. 
	
	In Bayesian inference, new information is contained in  data or experiments. To the contrary, we are interested in information that is contained in an expectation. The principle that allows us to address this type of information is that of minimum relative entropy or discrimination information \cite{kullback1951information,jaynes1986background,jaynes1982rationale}. The principle states that the unique posterior, $\pi$, living in the space of probability distributions, $\mathcal{P}$, constrained by an expectation of the form $\expect{\pi}[g]=\mu$, is the one that is hardest to discriminate from the prior, $\rho$. Equivalently, it is the one that minimises their relative entropy. Mathematically, this gives rise to a variational optimisation problem of the following form
		\begin{align}
		\label{eq:EI}
		\min_{\pi\in\mathcal{P}} ~ \underbrace{\kullbacks{\pi}{\rho}}_{\text{relative entropy}} 	\text{ s.t. } \underbrace{\mu = \expect{\pi}\left[g \right]}_{\text{new information}}  
		\end{align}
	where $\kullbacks{\pi}{\rho} = \expect{\pi}[\log \pi ] - \expect{\pi}[\log\rho]$. The solution is a Boltzmann distribution $\pi \propto e^{-\frac{1}{\lambda}g} \rho$ with $\lambda$ so that $\mu = \expect{\pi}\left[g \right]$.
	
	%	This is a special case of Theorem \ref{lem:det2}.
	
	\subsection{Entropic inference for optimisation}
	Mathematical optimisation addresses problems of the following form where $\mathcal{X}\subset \mathbb{R}^{n}$ represents the feasible subset.
		\begin{equation}
		\label{eq:opt}
		x^* = \arg \min_{x\in\mathcal{X}} q(x)
		\end{equation}
	
	Classic numerical optimisation strategies iterate a running estimate of the solution $x^*$, assimilating new information with each iteration. 
	
	Instead of focussing on a single estimate, we wish to model this search with a belief sequence $\vectorstyle{\pi} =\{\pi_g\}_g$. 	The more information that is assimilated, the more certain we get about the solution. To suit the purpose of optimisation, such a sequence should exhibit the following property
		\begin{equation}\label{eq:seq}
		\begin{aligned}
		&\lim_{g\rightarrow\infty}  \expect{\pi_g}[x] = \arg\min_{x\in\mathcal{X}} q = x^*
		\end{aligned}
		\end{equation}
	
	Second, we need an inference procedure that facilitates an update operation based on some form of new information
		\begin{equation*}
		\pi_{g+1} \leftarrow \pi_g
		\end{equation*}
	
	Our approach to arrive at such an inference procedure tailored to optimisation is straightforward. A priori no information is available. We can represent this situation mathematically by encoding our uncertainty about the solution in a prior probability density function, $\rho$\footnote{If no information is available, we can choose $\rho \leftarrow \mathcal{U}_\mathcal{X}$.}. In order to gradually decrease our uncertainty, we then look for a posterior probability density function, $\pi$, that discriminates the least from our initial guess but has an expectation over the objective function that produces a lower estimate than the prior expectation.
		\begin{equation}\label{eq:eo}
		\begin{aligned}
		\min_{\pi\in\mathcal{P}} ~& \kullbacks{\pi}{\rho}  \\
		\text{s.t. } &  \expect{\pi}[q] \leq \expect{\rho}[q] - \Delta, ~\Delta > 0
		\end{aligned}
		\end{equation}
	
	Whether a solution exists depends on the parameter $\Delta$. If it exists, the solution is a Boltzmann distribution similar to (\ref{eq:EI}). It can be shown that if we choose $\lambda > 0$, there also exists some $\Delta(\rho,\lambda) > 0$ \cite{luo2019minima}. For now, this is sufficient.
		\begin{equation}\label{eq:solEOC}
		\pi (x) \propto \rho(x) \cdot e^{-\lambda q(x)}
		\end{equation}
	
	Clearly, these expressions bear close correspondence with the classical Bayesian update, generating a posterior, $\pi$, by multiplying a prior, $\rho$, with an expression encoding a \textit{likelihood}, which is the likelihood of optimality in this case. By applying substitutions $\pi\leftarrow\pi_{g+1}$ and $\rho\leftarrow\pi_g$, we can establish the sequence $\vectorstyle{\pi}$. This sequence has exactly property (\ref{eq:seq}) \cite{luo2019minima}.  
	
	This property basically implies that the sequence converges to the Dirac delta distribution centred at the optimum. Since the next idea is to emulate the behaviour of this distribution numerically, the property also implies that the sample set will become more and more localised. Although this is exactly what we want from a theoretical point of view, this is troublesome from an algorithmic perspective. {Therefore, we desire to encode a second prior into the sequence that stimulates exploration. We introduce the augmented entropic optimisation problem. Here $\mathcal{U}_\mathcal{X}$ represents the uniform on $\mathcal{X}$ and $\alpha$ is a scaling factor that allows us to attribute more importance to either prior.}
		\begin{equation}\label{eq:eo2}
		\begin{aligned}
		\min_{\pi\in\mathcal{P}} ~&\alpha  \kullbacks{\pi}{\pi_g} + (1-\alpha)  \kullbacks{\pi}{\mathcal{U}_\mathcal{X}}, 0 < \alpha < 1   \\
		\text{s.t. } &  \expect{\pi}[q] \leq \expect{\pi_g}[q] - \Delta, ~\Delta >0
		\end{aligned}
		\end{equation}
	
	For some $\Delta(\pi_g,\lambda,\alpha)>0$, the solution is given by\footnote{For notational convenience, we absorb the uniform distribution $\mathcal{U}_\mathcal{X}$ into the proportionality. Note that $\pi_{g+1}$ is therefore zero outside the set $\mathcal{X}$.} \cite{lefebvre2020elsoc}
		\begin{equation}\label{eq:solEOC2}
		\begin{aligned}
		\pi_{g+1} \propto \pi_g^\alpha \cdot \mathcal{U}_\mathcal{X}^{1-\alpha} \cdot e^{-\lambda q}\propto \pi_g^\alpha\cdot e^{-\lambda q}
		\end{aligned}
		\end{equation}
	and it can be verified that 
		\begin{equation}
		\label{eq:sequence}
		\lim_{\alpha\rightarrow 1} \lim_{g\rightarrow\infty} \pi_g \propto \lim_{\alpha\rightarrow 1} e^{-\frac{\lambda}{1-\alpha} q} \propto \delta(x-x^*)
		\end{equation}
	
	By emulating this sequence numerically, we can construct a stochastic search algorithm tailored to the problem (\ref{eq:opt}).
	
	\subsection{Stochastic search methods}\label{sec:stochastic-search-methods}
	
	The idea is to update a parametrised belief and match its empirical distribution features with those of the theoretical distribution $\pi_g$. In practice, we use a parametric density function $\pi_\theta$ with $\theta \in\Theta$ (e.g., $\mathcal{N}(\mu,\Sigma)$) to approximate the entities in $\vectorstyle{\pi}$. We derive the associated parameter sequence, $\vectorstyle{\theta}$, from samples. Therefore, we project the theoretical entity $\pi_{g+1}$ onto the density space generated by $\pi_\Theta$, minimising their relative entropy. The objective is then manipulated into an expectation over $\pi_g$, which we estimate by sampling $\pi_{\theta_g}$.
		\begin{align*}
		\theta_{g+1} &= \arg\min_{\theta\in\Theta}\kullbacks{\pi_{g+1}}{\pi_\theta} \\ 
		&\approx \arg \max_{\theta\in\Theta} \hat{\mathbb{E}}_{\mathcal{D}_g} \left[   e^{-\lambda (q+(1-\alpha) \log \pi_{\theta_g})} \log \pi_\theta  \right] 
		\end{align*}
	where $\mathcal{D}_g = \{x_{g,j}\}_j,x_{g,j} \sim \pi(\theta_g)$. For details, see appx.  \ref{app:stochastic-search-algorithms}.
	
	A basic implementation is presented in Alg. \ref{alg:2}. Clearly, the computation bears close correspondence with Alg. \ref{alg:1}, though it remains unclear what underlying principle is mutual. For further exploration of similar ideas, we refer to \cite{abdolmaleki2015model,abdolmaleki2017deriving}.

	\begin{algorithm}[t!]
		\algsetup{linenosize=\small}
		\small
		\caption{Basic stochastic search method}
		\begin{algorithmic}[1]
			\label{algo:ES}
			\STATE \textbf{input }$\{\mu_0,\Sigma_0\}$
			\STATE \textbf{output }$\{\mu_g,\Sigma_g\}$
			\FOR {$g=0,1,2,\dots$}
			\FOR {$j=1,2,3,\dots,M$}
			\STATE  $x_j \sim \mathcal{N}\left(\mu_g,{\Sigma}_g\right)$ 
			\STATE  $w_{j} \leftarrow \exp(-({\lambda} q(x_j ) + {(1-\alpha)} \log \mathcal{N}(x_j |\mu_g,\Sigma_g)))$
			\ENDFOR
			\STATE $\mu_{g+1} = \sum\nolimits_j \tfrac{w_j}{\sum_j w_j} x_j  $
			\STATE $\Sigma_{g+1} = \sum\nolimits_j \tfrac{w_j}{\sum_j w_j} (x_j - \mu_{g+1}) (x_j - \mu_{g+1})^\top$
			\ENDFOR
		\end{algorithmic}
		\label{alg:2}
	\end{algorithm}
	
	\section{Entropic deterministic optimal control}
	{In this section we answer question 2). We provide an original derivation for a sample based trajectory optimiser that demonstrates similar heuristics as were described in sec. \ref{sec:control-as-inference-through-lsoc}. Contrary to MPPI contrary our derivation follows directly from of the ideas in sec. \ref{sec:entropic-optimisation} applied to the problem in sec. \ref{sec:problem-description}.}
	
	Although entropy regularisation is a well-known concept in reinforcement learning, it has only been applied to stochastic problems that are naturally embedded in a probabilistic framework. As we will show for deterministic optimal control, entropy regularisation gives rise to explicit expressions for the posterior policies in terms of a conditional expectation taken over trajectories induced by prior policies similar to the setting of LSOC. The crucial difference with Entropic Deterministic Optimal Control (EDOC) is that LSOC inherently solves a stochastic optimal control problem whilst EDOC still solves a deterministic optimal control problem. The inference in LSOC is facilitated by the input noise  inherent to the stochastic problem. Contrarily, in EDOC, the inference is put in place intentionally. The major consequence of this difference is that it disentangles the inference from the control.
	
	\subsection{Entropic Bellman equation}
	We consider the following EDOC problem, where $0<\Delta$ and $0<\alpha<1$.
		\begin{equation}\label{eq:edoc}
		\begin{aligned}
		\min_{\vectorstyle{\pi}\in\mathcal{P}} ~&\alpha  \kullbacks{p(\vectorstyle{\pi})}{p(\vectorstyle{\pi}_g)} + (1-\alpha)  \kullbacks{p(\vectorstyle{\pi})}{p(\vectorstyle{\mathcal{U}}_\mathcal{A})}  \\
		\text{s.t. } &  \expect{p(\vectorstyle{\pi})}[R] \leq \expect{p(\vectorstyle{\pi}_g)}[R] - \Delta
		\end{aligned}
		\end{equation}
	
	Starting from problem (\ref{eq:dtSOC}), we reason as follows. The optimisation variables are given by the control sequence $\vectorstyle{a}$. {This problem is no different than problem (\ref{eq:opt}), so we could address it with Alg. \ref{alg:2}}\footnote{ Note that in this case, the weights $w_j$ would not be time dependent.}. However, we would instead like to exploit the dynamic structure of the problem. {Although the solution of (\ref{eq:dtSOC}) is indeed a trajectory, from (\ref{def:SBE}), we know that a sequence of optimal policies, $\mathbf{a}^* = \{a^*_n(s)\}_n$, underpins this trajectory, which we only happen to know evaluated over deterministic dynamics.} Therefore, we introduce the conditional policy beliefs, $\boldsymbol{\pi}_g = \{\pi_{g,n}(a|s)\}_n$. These beliefs express our uncertainty about control $a$ given state $s$ at time $n$ for the $g$-th iteration. Second, we define a trajectory probability density function conditioned on $s_n$ and $\vectorstyle{\pi}_n$.  
		\begin{align*}
		p_n(\vectorstyle{\pi}_n) &= p(\boldsymbol{\tau}_{n}|s_n,\vectorstyle{\pi}_n) \\
		&= \prod\nolimits_{n=n}^{N-1} \delta(s_{n'+1}-f_{n'}(s_{n'},a_{n'})){\pi}_n(a_n|s_n)
		\end{align*}

	This definition allows us to cast (\ref{eq:dtSOC}) in a probabilistic manner. 
		\begin{equation*}
		\label{eq:E1}
		V_n(s_n) = \min_{\vectorstyle{a}_n} \expect{p_n(\vectorstyle{\pi}_n)}[R_n(\boldsymbol{\tau}_n)]
		\end{equation*}
	In the case that the dynamics are deterministic, this probabilistic problem is completely equivalent to the original problem and solves for a Dirac distribution or deterministic policy. In turn, this probabilistic model motivates problem (\ref{eq:edoc}). 
	
	We will address the solution to this problem in the next paragraph. First, we further motivate our problem definition by demonstrating that we would have arrived at the same result if we had regularised the Bellman equation (\ref{def:SBE}) instead. The following lemma certifies the consistency of the proposed regularisation. It also emphasises that we only regularise over the optimisation variables, $\boldsymbol{a}$, but not the dynamics. Finally, it introduces the \textit{entropic Bellman equation}.
	\begin{lemma} The solution of problem (\ref{eq:edoc}) is governed by the following entropic Bellman equation
			\begin{equation}
			\label{def:EBE}
			\begin{aligned}
			V_{g+1,n} &= \min_{\pi\in\mathcal{P}} \expect{\pi}\left[Q_{g+1,n} \right] + \kullbacks{\pi}{\pi_{g,n}^\alpha \cdot \mathcal{U}_\mathcal{A}^{1-\alpha}} \\
			\pi_{g+1,n} &= \arg\min_{\pi \in\mathcal{P}} \expect{\pi}\left[Q_{g+1,n} \right] + \kullbacks{\pi}{\pi_{g,n}^\alpha \cdot  \mathcal{U}_\mathcal{A}^{1-\alpha}}
			\end{aligned}
			\end{equation}
		where $Q_{g+1,n} = \lambda r_n + V_{g+1,n+1}$.
	\end{lemma}

	\begin{proof} First, by introducing multiplier $\lambda > 0$, we can recast (\ref{eq:edoc}). Note that the dynamics cancel out in the fractions. 
			\begin{equation*}
			\min_{\boldsymbol{\pi}_n\in\mathcal{P}} \expect{p_n(\boldsymbol{\pi}_n)}\left[\lambda R_n + \alpha \log \tfrac{\vectorstyle{\pi}_n}{\vectorstyle{\pi}_{g,n}} + (1-\alpha) \log \tfrac{\vectorstyle{\pi}_n}{\vectorstyle{\mathcal{U}}_{\mathcal{A},n}} \right]
			\end{equation*}
		
		Then, writing out the objective, we retrieve an optimisation problem that adheres  to Bellman's principle of optimality
			\begin{multline*}
			\label{eq:proof}
			\min_{\boldsymbol{\pi}_n\in\mathcal{P}} \expect{p_n(\boldsymbol{\pi}_n)}\Big[\lambda r_N + \dots \\   \sum\nolimits_{n'=n}^{N-1} \left(\lambda r_{n'} + \alpha \log \tfrac{\pi_{n'}}{\pi_{g,n'}} + (1-\alpha) \log \tfrac{\pi_{n'} }{\mathcal{U}_\mathcal{A}} \right) \Big] \\ 
			= \min_{\pi_n \in\mathcal{P}} \expect{\pi_n}\Big[\lambda r_n + \alpha \log \tfrac{\pi_{n}}{\pi_{g,n}} + (1-\alpha) \log \tfrac{\pi_n}{\mathcal{U}_\mathcal{A}} +  \dots \\
			\min_{\vectorstyle{\pi}_{n+1}\in\mathcal{P}} \expect{p_{n+1}(\vectorstyle{\pi}_{n+1})}\Big[\dots \\ 
			\lambda R_{n+1} + \alpha \log \tfrac{\boldsymbol{\pi}_{n+1}}{\boldsymbol{\pi}_{g,n+1}} + (1-\alpha) \tfrac{\boldsymbol{\pi}_{n+1}}{\boldsymbol{\mathcal{U}}_{\mathcal{A},n+1}} \Big]\Big] \\ 
			= \min_{\pi_n\in\mathcal{P}} \expect{\pi_n}\Big[\lambda r_{n} +  \alpha \log \tfrac{\pi_n}{\pi_{g,n}} + (1-\alpha) \log \tfrac{\pi_n}{\mathcal{U}_\mathcal{A}}  + V_{g+1,n+1} \Big] 
			\end{multline*}	
		where we have defined 
			\begin{equation*}
			V_{g+1,n+1} = \min_{\vectorstyle{\pi}_{n+1}\in\mathcal{P}} \expect{p_{n+1}(\vectorstyle{\pi}_{n+1})}\left[ 
			\lambda R_{n+1} + \log \tfrac{\boldsymbol{\pi}_{n+1}}{\boldsymbol{\pi}_{g,n+1}^\alpha \boldsymbol{\mathcal{U}}_{\mathcal{A},n+1}^{1-\alpha}}\right]
			\end{equation*}
		The lemma follows.
	\end{proof}
	
	Henceforth, we will treat $\lambda > 0$ as a hyper-parameter.
	
	\subsection{Path integral solution} 
	In this section, we address the tractability of the EDOC problems in (\ref{eq:edoc}) and (\ref{def:EBE}). First, we establish a recurrence relation for the optimal policy belief sequence similar to that done for problem (\ref{eq:eo2}) with equation (\ref{eq:solEOC2}). Second, we illustrate that this solution gives rise to an explicit path integral expression for the optimal posterior policy belief. 
	
	The first result is summarised by the following theorem. For the proof, we refer to appendix \ref{sec:proof-of-theorem-reflemdet2}. This result is known as the soft Bellman equation and has been studied in combination with stochastic system dynamics, e.g., in \cite{rawlik2013stochastic} for $\alpha = 0$ and in \cite{levine2018reinforcement} for $\alpha = 1$. The second, which is novel, is summarised in the corollary beneath it.
	
	\begin{theorem}\label{lem:det2}
		The solution of problem (\ref{def:EBE}) is given by 
			\begin{equation*}
			\begin{aligned}
			\pi_{g+1,n} &\propto \pi_{g,n}^{\alpha} \exp\left(- Q_{g+1,n}+V_{g+1,n} \right) \\
			V_{g+1,n} &= -\log \int \pi_{g,n}^{\alpha}\exp\left(- Q_{g+1,n}\right) \text{d}a 
			\end{aligned}
			\end{equation*}
	\end{theorem}
	
	Theorem \ref{lem:det2} points out that, similarly to equation (\ref{eq:solEOC2}), the relationship between the posterior and prior optimal policy belief functions ($\boldsymbol{\pi}_{g+1}$ and $\boldsymbol{\pi}_{g}$, respectively) is governed by a Bayesian type recurrence relation. However, in this case, the recursion still depends on the functions $Q_{g+1,n}$ and $V_{g+1,n}$ as a consequence of the problem's sequential nature. This is in contrast to (\ref{eq:solEOC2}), which only depends on the objective $q$.
	
	Fortunately, the results can be further developed into an explicit expression for the posterior optimal policy belief function that depends solely on prior information contained in $\boldsymbol{\pi}_{g}$. This result is a direct consequence of the \textit{deterministic} system dynamics and is summarised by the following corollary.
	
	\begin{corollary}\label{lem:3} The posterior optimal policy belief function $\pi_{g+1,n}$ can be expressed as
			\begin{equation*}
			\pi_{g+1,n}(a|s) = \pi_{g,n}(a|s) e^{-r_{g,n}} \frac{Z_{g+1,n+1}(f_n(s,a))}{Z_{g+1,n}(s)}
			\end{equation*}
		where 
			\begin{equation*}
			\begin{aligned}
			r_{g,n} &= \lambda r_n + (1-\alpha) \log \pi_{g,n} \\
			Z_{g+1,n} &= \exp(- V_{g+1,n})
			\end{aligned}
			\end{equation*}
					and 
			\begin{equation*}
			\begin{aligned}
			Z_{g+1,n}(s) &= \expect{p_n(\boldsymbol{\pi}_{g,n})} \left[\exp(- R_{g,n})\right] \\
			p_n(\boldsymbol{\pi}_{g+1,n})&\propto p_n(\boldsymbol{\pi}_{g,n}) \exp(- R_{g,n}) \\
			R_{g,n} &= \lambda r_N + \sum\nolimits_{n'=n}^{N-1} r_{g,n'} 
			\end{aligned}
			\end{equation*}
	\end{corollary}
	\begin{proof} 
		First, we define the function $Z_{g+1,n}$ as in the theorem. Substituting the transformation into  the solution of Theorem \ref{lem:det2} yields the following recurrence relation for $Z_{g+1,n}$ (this is only because $Z_{g+1,n+1}$ is evaluated \textit{\underline{deterministically}} in $s_{n+1}$).
			\begin{align*}
			Z_{g+1,n} &= \int \pi_{g,n}^{\alpha}e^{-\lambda  r_n}Z_{g+1,n+1} \text{d}a  \\
			&= \expect{\pi_{g,n}} \left[e^{- r_{g,n}} Z_{g+1,n+1}\right]
			\end{align*}
		
		We can then easily develop this recursion into an explicit expression for $Z_{g+1,n}$ as a conditional expectation over the path probability initialised at $n$, that is, $p_n(\boldsymbol{\pi}_{g,n})$. Recall that this probability is conditioned on $s_n$, so $Z_{g+1,n}$ is indeed a function of it, which proves the expression for $Z_{g+1,n}$. 
			\begin{equation*}
			Z_{g+1,n} = \expect{p_n(\boldsymbol{\pi}_{g,n})} \left[\exp(- R_{g,n})\right]
			\end{equation*}
		
		Now we can substitute this result into the recursive expression for $\pi_{g+1,n}$ and recover the main result. 
		
		Finally, one can substitute the explicit expression for $\pi_{g+1,n}$ into  the trajectory probability, $p_n(\boldsymbol{\pi}_{g+1})$. 
			\begin{equation*}
			\tfrac{p_n(\boldsymbol{\pi}_{g+1,n}) }{p_n(\boldsymbol{\pi}_{g,n}) }=  \tfrac{{Z_{g+1,n+1}}}{Z_{g+1,n}}\tfrac{{Z_{g+1,n+2}}}{{Z_{g+1,n+1}}} \cdots  \tfrac{{Z_{g+1,N-1}}}{{Z_{g+1,N-2}}}\tfrac{1}{{Z_{g+1,N-1}}} e^{- R_{g,n} }
			\end{equation*}
			
		It follows that the quotients cancel out when we multiply over the entire trajectory expect for $Z_{g+1,n}$ which in fact normalizes the trajectory density function. 
	\end{proof}
	
	The corollary above implies that the {function}, $Z_{g+1,n}$, and hence the posterior policy belief, $\pi_{g+1,n}$, can be quantified explicitly by evaluating conditional expectations over the prior path probabilities, $p_n(\boldsymbol{\pi}_{g,n})$. Although these results follow quite naturally from Theorem \ref{lem:det2}, they have important consequences in terms of the tractability and computability that are unique to the entropic regularisation of deterministic (rather than stochastic)  optimal control. As noted, the framework collapses when aside from the purposeful uncertainty, stochasticity is introduced on account of the dynamics. { In this case, an expectation over the stochastic dynamics emerges in the definition of $Q_{g,n}$, i.e., $\lambda r_n(s,a) + \expect{p(s'|s,a)}[V_{g,n+1}(s')]$ instead of $\lambda r_n + V_{g,n+1}(f_n(s,a))$. As a result, it is impossible to develop the recursion from Theorem \ref{lem:det2} into explicit expressions because we cannot get rid of the expectation in the exponent.} 
	
	Clearly, these expressions also bear a close correspondence with LSOC. We note that the Feynman-Kac formula establishes a relation between certain partial differential equations and stochastic processes. In particular, it expresses the solution of a partial differential equation as a conditional expectation. The association with stochastic processes emerges, as they constitute a framework where such conditional expectations arise naturally. In our work, the conditional expectation is not associated to any stochastic process, but rather to the Bayesian policy beliefs. Put differently we compute what we might refer to as a \textit{Bayesian path integral} instead of a stochastic path integral. The difference lies in that the conditional expectation is taken over {purposeful} uncertainty introduced to construct a consistent inference procedure about the underlying deterministic optimal control problem; it is not inherent to the problem. 
	
	\subsection{Entropic MPPI}\label{sec:entropic-path-integral-control}
	{ From corollary (\ref{lem:3}), we can derive a sample-based trajectory optimiser akin to MPPI. The algorithm demonstrates similar heuristics to those discussed in sec. \ref{sec:control-as-inference-through-lsoc}, amongst novel attributes. Although the algorithm's structure closely relates to Alg. \ref{alg:1}, its derivation stems from an entirely different theoretical context.} { As a disclaimer, we note that our derivation seeks similarities with Alg. \ref{alg:1} intentionally and therefore uses a restricted class of parametric policy beliefs. However, we note that corollary (\ref{lem:3}) does not in fact exclude the use of any policy, so one could also consider using a Gaussian mixture to address certain problems, e.g., those with many local optima.}	
	
	In particular, we approximate the posterior policy beliefs, $\{\pi_{g+1,n}\}_n$, with a parametric distribution and infer the associated parameters from samples. In this case, we apply the prior policy belief sequence, $\{\pi_{g,n}\}_n$, for a given and fixed initial state, $s^*_0$, and collect data in the form of sample paths or trajectories, $\mathcal{D}_g=\{\vectorstyle{\tau}_j\}$. This means that we will sample from the policies as if they were probability density functions rather than Bayesian belief functions. Averaging over these sample trajectories then allows us to evaluate the expressions in corollary \ref{lem:3}. Given the clear resemblance with Alg. \ref{alg:1}, we can reasonably refer to it as Entropic MPPI (EMPPI); Alg. \ref{alg:3}.
	
	\subsubsection{Policy belief parametrisation} 
	Pursuing general tractability, we are interested in locally linear Gaussian policies with temporal distribution parameters, $\{\theta_{g,n}\}_n$, where $\theta_{g,n} = \{k_{g,n},\matrixstyle{K}_{g,n},\Sigma_{g,n}\}$. With the exception of the  covariance, this is similar to imposing a piecewise linear controller. Since the normal is unimodal, this approximation renders the stochastic search method locally similar to gradient-based approaches.
		\begin{equation*}
		\pi_{g,n}(a|s) \approx \pi_n(a|s;\theta_g) = \mathcal{N}(a|k_{g,n} + \matrixstyle{K}_{g,n} s,\Sigma_{g,n})
		\end{equation*}					
		
	\subsubsection{Projection strategy}
	Since the beliefs are conditioned on $s$, we extend the projection idea with an expectation over the available samples for that time instant, $p(s_n|s_0,\boldsymbol{\pi}_g)$. This expression denotes the probability of state $s_n$ conditioned on the initial state $s_0$ and the prior policy belief sequence $\boldsymbol{\pi}_g$. 
		\begin{align*}
		\theta_{g,n} &= \arg \min_\theta \expect{p{(s_n|s_0,\boldsymbol{\pi}^g)}}\left[\kullbacks{\pi_{g+1,n}}{\pi_\theta}\right] \\
		&= 	\arg\max_\theta \sum\nolimits_j \tfrac{w_{n,j}}{\sum_j w_{n,j}} \log \pi (a_{n,j}|s_{n,j},\theta)
		\end{align*}				
		with 							
			\begin{equation}
		\label{eq:weight}
		- \log w_{n,j} = R_{g,n,j} = \lambda R_{n,j}  + (1-\alpha) \sum\nolimits_{n'=n}^{N-1} \log \pi_{g,n',j}
		\end{equation}		
	Apart from the cost, $R_n$, the weights also contain a second term $\sum_{n'=n}^{N-1} \log \pi_{g,n',j}$. {Given that $\log \pi_{g,n,j} < \log \pi_{g,n,j'}$ if $j$ was less likely than $j'$, this auxiliary cost  encourages unlikely whilst penalizing likely actions.} {Note that if we had not augmented the entropic optimisation problem with a second prior this term would vanish ($\alpha = 1$).} This optimisation problem is solved most efficiently by calculating a likelihood weighted estimate of the joint Gaussian distribution, $\mathcal{N}(\tau_{g+1,n}|\hat{\mu}_{\tau,g+1,n},\hat{\Sigma}_{\tau\tau,g+1,n})$. Parameters $\{k_{g+1,n},\matrixstyle{K}_{g+1,n},\Sigma_{g+1,n}\}$ are then found by conditioning on $s$. Further note that we solve this problem for every $n$
	\begin{equation}
		\label{eq:update}
		\begin{aligned}
		\matrixstyle{K}_{g+1,n} &= \hat{\Sigma}_{as,g+1,n} \hat{\Sigma}_{ss,g+1,n}^{-1} \\
		k_{g+1,n} &= \hat{\mu}_{a,g+1,n} - \matrixstyle{K}_{g+1,n}\hat{\mu}_{s,g+1,n} \\
		\Sigma_{g+1,n} &= \hat{\Sigma}_{aa,g+1,n} - \matrixstyle{K}_{g+1,n} \hat{\Sigma}_{ss,g+1,n} \matrixstyle{K}_{g+1,n}^\top
		\end{aligned}
		\end{equation}					
	where 							
		\begin{align*}
		\hat{\mu}_{\tau,g+1,n} &= \left\langle \tau_{g,n,j} \right\rangle \\
		\hat{\Sigma}_{\tau\tau,g+1,n} &= \left\langle \left(\tau_{g,n,j} - \hat{\mu}_{\tau,g+1,n}\right)\left(\tau_{g,n,j} - \hat{\mu}_{\tau,g+1,n}\right)^\top \right\rangle
		\end{align*}							
	with $\langle (\cdot) \rangle = \sum_j \frac{w_{n,j}}{\sum_j w_{n,j}} (\cdot) $.

	{Similar algorithms are described in references \cite{williams2016agressive,williams2017model,williams2018information,bhardwaj2021fast,rajamaki2016sampled,ghandi2021ieeeral,lefebvre2019path}}. { Amongst these MPPI implementations,} the presence of feedback, the presence of the term $\sum_{n'=n}^{N-1} \log \pi_{g,n',j}$ in the likelihood weights and the applicability to general deterministic optimal control problems are, as far as we are aware, unique to Alg. \ref{alg:3}, {and are now theoretically justified.}
	
	\subsubsection{Numerical implementation}
	{To increase the overall numerical stability, we first use exponential smoothing}
\begin{equation}
		\label{eq:updateex}
		\begin{aligned}
		\theta_{g+1,n} &\leftarrow   \beta \theta_{g+1,n} + (1-\beta) \theta_{g,n}
		\end{aligned}
		\end{equation}						
	Second the use of Monte Carlo estimates implies that $M\gg 1$. {To estimate the covariance matrix of the trajectory, $\Sigma_{\tau\tau,N-1}$, we need at least a multiple of $(N-1)(\frac{1}{2}(n_\tau^2+\tfrac{1}{2}n_\tau)$ samples.} It is clear that the updates are prone to high variance for finite $M$. To remedy this issue we project the time signals $\{\hat{\mu}_{\tau,g+1,n},\hat{\Sigma}_{\tau\tau,g+1,n}\}$ on a polynomial space of order $d$ spanned by the basis $\{1,n,n^2,\dots,n^d\}$. Despite this measure the update for $\matrixstyle{K}_{g+1,n}$ remained too unstable. Hence instead we substitute a fixed gain matrix $\matrixstyle{K}$ for $\matrixstyle{K}_{g+1,n}$ in the updates in (\ref{eq:update}). Finally, since $\Sigma_{g+1,n}\in\mathbb{S}_+^n$, an auxiliary procedure is used that changes $\gamma$ until $\Sigma_{g+1,n}\leftarrow \Sigma_{g+1,n} + \gamma \matrixstyle{I} \in\mathbb{S}_+^n$.  
	
	\subsubsection{Relation to other algorithms} 
	Standard gradient-based trajectory optimisation algorithms iterate between a forward and a backward pass to probe the local problem geometry about the iterate trajectory and to optimise it, respectively. Similarly, Alg. \ref{alg:1} and \ref{alg:3} iterate between a forward Monte Carlo step and an inference step. As opposed to the backward step in gradient-based algorithms, the inference step does not possess a causal structure and is thus  amenable to parallelisation . In Alg. \ref{alg:1} and \ref{alg:3}, the weights are time dependent, so in some way, the recursive nature of the problem emerges. {This is opposed to Alg. \ref{alg:2} and other stochastic search algorithms such as CMA-ES\cite{abdolmaleki2017deriving}, which could be used to solve (\ref{eq:dtSOC}) as well. The use of conditional policies is not straightforward, nor are the weights time dependent. So taking into account the blue print architecture of Alg. \ref{alg:2}, EMPPI can be understood as a temporal rollout of an Evolutionary Strategy or a stochastic implementation of gradient-based trajectory optimisation algorithms where gradient information is inferred from the sampled trajectories.} Finally, we emphasise again that Alg. \ref{alg:1} solves specific stochastic optimal control problem (\ref{eq=stoc}) whilst Alg. \ref{alg:3} solves general deterministic optimal control problem (\ref{eq:dtSOC}). 
	
	\begin{algorithm}[t!]
		\algsetup{linenosize=\small}
		\small
		\caption{Entropic MPPI}
		\begin{algorithmic}[1]
			\label{algo:PI2-DDP}
			\STATE \textbf{input} $\lambda$, $\alpha$, $\{k_{0,n},\matrixstyle{K}_{0,n},\Sigma_{0,n}\}_n$
			\STATE \textbf{output }$\{{k}_{g,n},\matrixstyle{K}_{g,n},\Sigma_{g,n}\}_n$
			\FOR {$g=0,1,2,\dots$}
			\FOR {$j=1,2,3,\dots,M$}
			\FOR {$n=0,1,2,\dots,N-1$}
			\STATE  ${a}_{n,j}\sim \mathcal{N}\left({k}_{g,n} + \matrixstyle{K}_{g,n} s_{n,j},{\Sigma}_{g,n}\right)$ 
			\STATE  ${s}_{n+1,j} \leftarrow {f}_n\left({s}_{n,j},{a}_{n,j}\right)$
			\STATE  $r_{g,n,j} \leftarrow \lambda r_{n}(s_{n,j},a_{n,j}) + (1-\alpha) \log \pi_{g,n}(a_{n,j}|s_{n,j}) $
			\ENDFOR
			\STATE  $r_{g,N,j} \leftarrow \lambda r_N(s_{N,j})$
			\ENDFOR
			\FOR {$n=0,1,2,\dots,N-1$}
			\STATE $\{k_{g+1,n},\matrixstyle{K}_{g+1,n},\Sigma_{g+1,n}\} \leftarrow$ (\ref{eq:weight}),(\ref{eq:update}),(\ref{eq:updateex}) 
			\ENDFOR
			\ENDFOR
		\end{algorithmic}
		\label{alg:3}
	\end{algorithm}
	
	\subsection{Numerical example}
	We present results for numerical experiments with a 4 dimensional planar robot arm operating in an environment with a single obstacle. { This is not only to demonstrate the practical implications of our theoretical investigation (corollary 1 specifically), but also to investigate the effect of the different MPPI algorithms and changing their parameter settings (Alg. 3) on the exploration versus exploitation behaviour of the resulting distributions.}
	
	\subsubsection{Environment} 
	The environment consists of a planar pendulum with $4$ links of length $1$. Masses $(m=1)$ are concentrated at the end of each link. We use torque inputs directly instead of generating kinematic trajectories and relying on low-level controllers. An OCP is formulated with horizon $N=25$. We use a relatively coarse time discretisation $\Delta t = 0.1$. The cost rate function is defined as $r_n= 0.1 \|a_n\|^2\Delta t +  10  \|\dot{q}\|^2\Delta t + \|q_{2:4}\|^2\Delta t$, penalising the energy consumption and aggressive moves. We intentionally do not encode information about the non-linearity of the dynamics nor the obstacles. The system is drawn towards a final end-effector configuration using the final cost term $r_N = 300 \log(\|d\|+0.1) + 10 \|d\|^2$, where $d$ represents the distance vector between the end-effector and goal configuration. We do not represent the obstacles in the cost function. Interactions with the obstacles are strictly through the dynamics; gradient-based algorithms are not in favour here. The contact dynamics are modelled through forces $F_c = \frac{d_c}{\|d_c\|^3} H(-d_c^\top \matrixstyle{J}_c\dot{q})$, where $H$ is the Heaviside function, $d_c$ is the distance vector between the obstacle and the near contact point and $\matrixstyle{J}_c$ is the Jacobian matrix computed at the nearest contact point. Described contact dynamics are inelastic. 
	
	\begin{figure}[t!]
		\centering
		\includegraphics[width=.865\columnwidth]{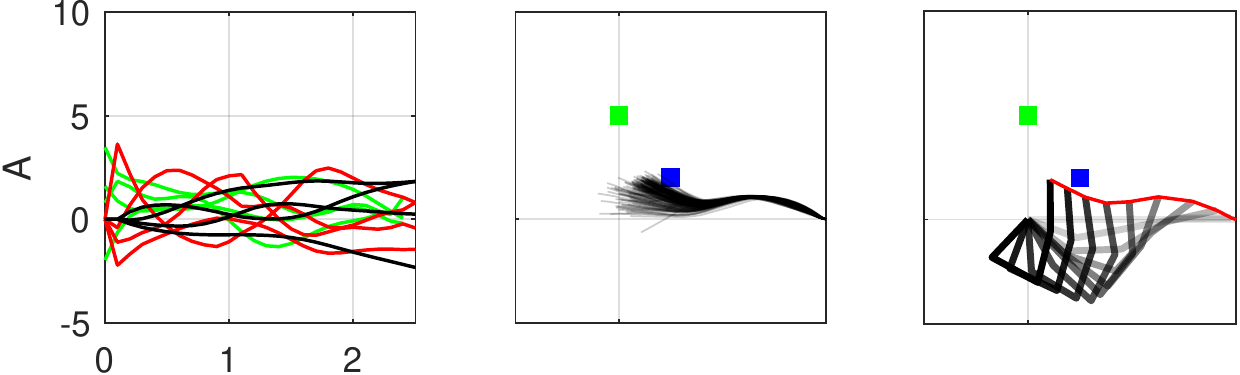}
		\includegraphics[width=.865\columnwidth]{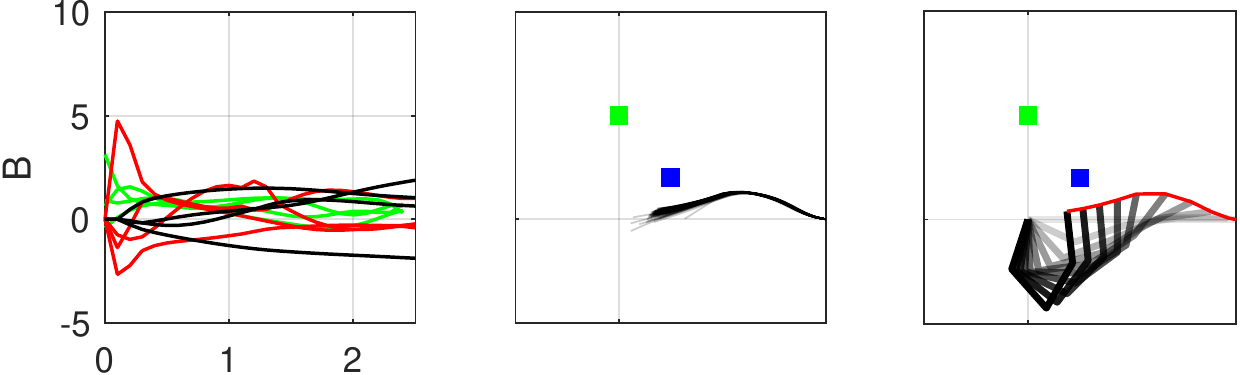}
		\hspace*{3.5pt}\includegraphics[width=.865\columnwidth]{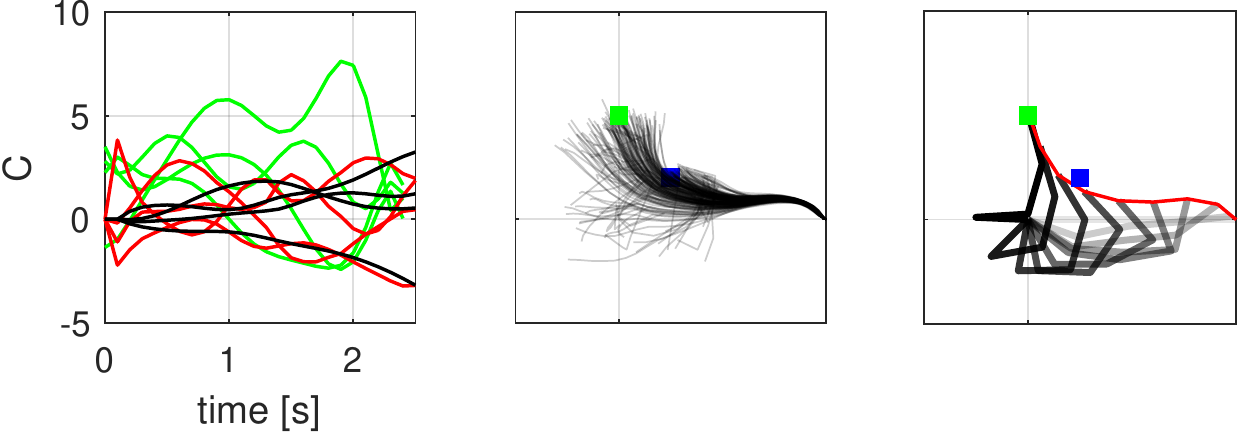}
		\vspace*{-3pt}
		\caption{\textit{From top to bottom}: Illustration of solutions obtained {with algorithms A, B and C}. \textit{From left to right}: Visualizations of the final policy after 200 generations. \textit{Left}: Deterministic time signals: joint angles (black), joint velocities (red) and joint torques (green). \textit{Middle}: Top view of 200 end-effector histories. \textit{Right}: Snap shots of deterministic system trajectory with end-effector trajectory in red. Goal configurations are shown in green, obstacles in blue.}
		\label{fig:EPIC}
	\end{figure}
	
	\subsubsection{Experiments} 
	We compare three versions of Alg. \ref{alg:3}. In version A, we set $\alpha = 1$, only update $k_{g,n}$, set $\matrixstyle{K}_{g,n} = 0$ and $\{\Sigma_{g,n}\} = 10^{-1} \matrixstyle{I}$. Version A is the most closely related to the MPPI implementation in \cite{williams2018information} { and therefore serves as a baseline}. An important improvement of MPPI comes from updating the covariance (deviating from the theory of LSOC) \cite{stulp2012path,bhardwaj2021fast}. As a result, it can be observed empirically that the covariance collapses prematurely. Therefore, in version B, we set $\alpha = 0$ but update $\Sigma_{g,n}$. According to (\ref{eq:sequence}), the policy belief functions will converge to a Dirac delta and it is anticipated that the search will converge prematurely. Version C implements the full algorithm. Versions A, B and C are initialised with the feedforward $\{k_{0,n}\} = 5\cdot10^{-1} $ and with covariances $\{\Sigma_{0,n}\} = 10^{-1} \matrixstyle{I}$. Unless specified, we set $(\lambda,\alpha,\beta,M) = (.2,.95,.1,200)$ and then run 200 generations.
	
	\subsubsection{Results} 
	The solution after $200$ generations is visualised in Fig. \ref{fig:EPIC}. Clearly, only C completes  the tasks successfully. As anticipated, we can observe the presence of premature distribution collapse for version B. The entropy of the distribution \textit{evaporates} and the search stalls. The performance of A is superior to that of B yet also it fails to execute the final reach of the complete manoeuvre {within 200 generations}. The main benefit of C over A is that it can automatically adapt the covariance of the policy. This allows the policy to discover interesting directions more rapidly. One can also observe that some of the alternative histories still collide with the obstacle; howevers the bulk effectively reaches the goal. These observations are confirmed by Fig. \ref{fig:EPIC2} and \ref{fig:EPIC3}. In particular, for version C, Fig. \ref{fig:EPIC3} clearly illustrates how the covariance of the policy self-adapts to the progress made on the problem.
	
	\begin{figure}[t!]
		\centering
		\includegraphics[width=.32\columnwidth]{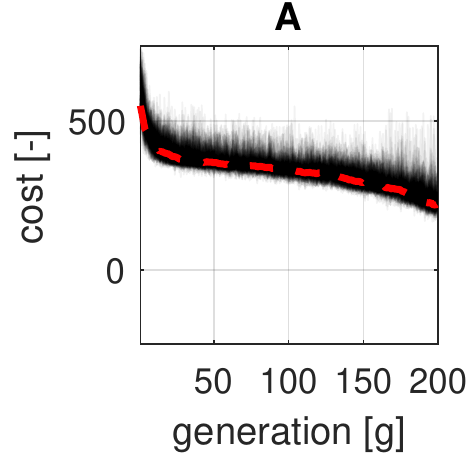}
		\includegraphics[width=.32\columnwidth]{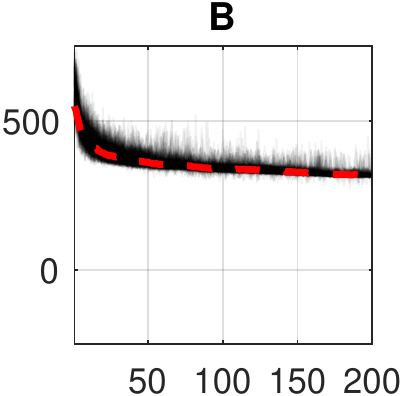}
		\includegraphics[width=.32\columnwidth]{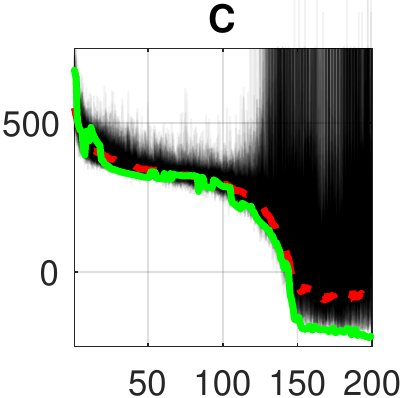}
		\caption{\textit{From left to right}: Convergence of sampled trajectory cost $R_0(\boldsymbol{\tau}_{g,j})$ (black) and soft mean $\hat{R}_{g,0} = -\frac{1}{\lambda} \log \frac{1}{M} \sum_j \exp(-\lambda R_0(\boldsymbol{\tau}_{g,j}))$ (dotted red) obtained with algorithms A, B or C. Premature convergence collapse occurs after 50 iterations with algorithm B. The entropy of the policies with algorithm C rises after about 75 iterations allowing to escape the local minimum after about 100 iterations. Due to the higher variance of the policies generated with algorithm $C$ some of the trajectories obtain a very high cost. However their contribution is filtered out by the exponential weighting as illustrated by the deterministic cost convergence depicted in green.}
		\label{fig:EPIC2}
		\centering
		\includegraphics[width=1\columnwidth]{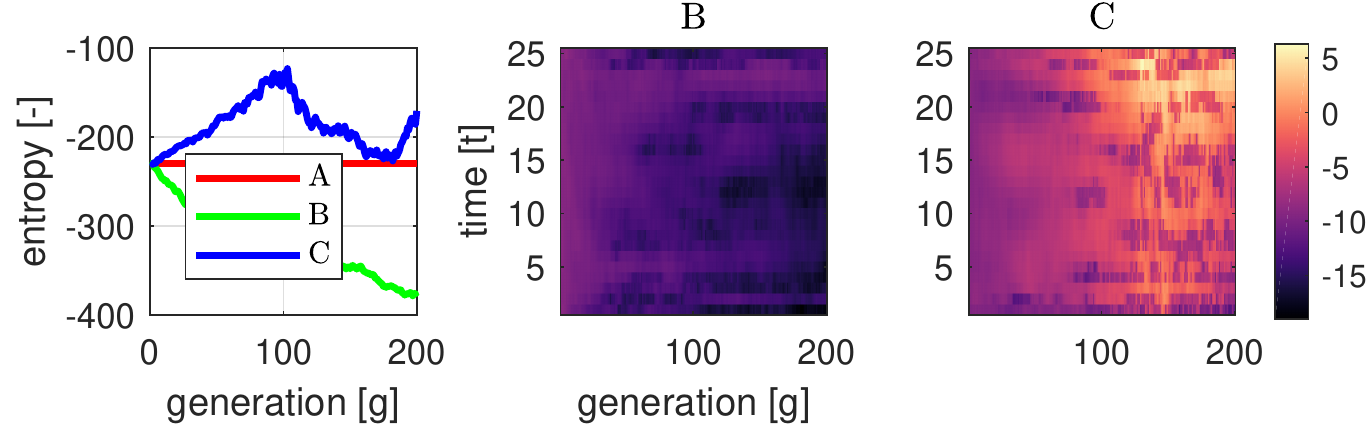}
		\vspace*{-12pt}
		\caption{\textit{Left}: Evolution of entropy content of the policy sequence $\{\vectorstyle{\pi}_g\}_g$ computed as $\sum_n \log |\Sigma_{g,n}|$. \textit{Right}: Distribution of entropy context of policies $\{\pi_{g,n}\}$  computed as $\log |\Sigma_{g,n}|$ as a function of time and iteration.}
		\label{fig:EPIC3}
	\end{figure}
	
	\section{Conclusion}
	Sample- based trajectory optimisation is a promising tool for robotics with complex and non-smooth dynamics and cost functions, both to synthesise complex behaviour and to compute real-time feedback controls. In this contribution, we have proposed an alternative derivation of the popular MPPI algorithm. Our derivation is founded on the framework of EDOC, an entropy-regularised version of the standard deterministic optimal control problem, for which we have shown that the optimal policy can be given by a policy belief function and expressed as a Bayesian path integral similarly to LSOC. We argue that our derivation allows for more principled future algorithmic development of sample-based trajectory optimisers and will become incorporated  into other challenges and applications.
	
	\section*{Acknowledgements} This research received funding from the ``Flemisch Artificial Intelligence Research (FlAIR)'' programme, Belgium.
	
	\newpage
	\appendices
	\section{Stochastic search algorithms}\label{app:stochastic-search-algorithms}
	We return to  the derivation in section \ref{sec:stochastic-search-methods}. Now, we wish to manipulate it into an expectation over the prior $\pi_g$, which would allow us to estimate it using Monte Carlo sampling. To arrive at the second line, we substitute the definition for the relative entropy. Since we optimise for $\theta$, we can further neglect the first term. Then, to arrive at the third line, we make use of the recurrence relation. This expectation is then approximated using a sample $\mathcal{D}_g = \{x_{j}\}_j$ where $x_j\sim\mathcal{\pi}_{\theta_g}$.
		\begin{align*}
		\theta_{g+1} &= \arg\min_{\theta\in\Theta}\kullbacks{\pi_{g+1}}{\pi_\theta} \\ 
		&= \arg \max_{\theta\in\Theta} \expect{\pi_g} [\pi_g^{1-\alpha}  e^{-\lambda q} \log \pi_\theta ]  \\
		&\approx \arg \max_{\theta\in\Theta} \hat{\mathbb{E}}_{\mathcal{D}_g} \left[e^{-\left(\lambda q + (1-\alpha)\log \pi_{\theta_g}\right)} \log \pi_\theta  \right] 
		\end{align*}								
	
	When we approximate beliefs $\pi_g$ using the Normal distribution $\mathcal{N}(x|\mu,\Sigma)$; this approach generates the updates shown in algorithm \ref{alg:2}. We refer to \cite{lefebvre2020elsoc} for further details.
	
	\section{Proof of Theorem \ref{lem:det2}}\label{sec:proof-of-theorem-reflemdet2}
	\begin{proof}
		Consider $\mathcal{L}=\expect{\pi}[l]$, where $\eta$ and $\lambda$ are multipliers associated to the normalisation and inequality constraints.
			\begin{equation*}
			\begin{aligned}
			\mathcal{L} &= \expect{\pi}\left[\lambda Q_{g+1,n} + \alpha\log\tfrac{\pi}{\pi_{g,n}} + (1-\alpha)\log\tfrac{\pi}{\mathcal{U}_\mathcal{A}} + \eta \right] - \eta 
			\end{aligned}
			\end{equation*}
		
		The partial derivative of the integrand is given by 
			\begin{equation*}
			\nabla_\pi l = \lambda Q_{g+1,n} + \alpha \log\tfrac{\pi}{\pi_{g,n}} + (1-\alpha) \log \tfrac{\pi}{\mathcal{U}_\mathcal{A}}+ \eta
			\end{equation*}							
		
		According to the fundamental lemma of the calculus of variations, the optimal policy belief function is then given by the root of the partial derivative of the integrand. 
	\begin{align*}
			\pi_{g+1,n} &\propto \pi_{g,n}^{\alpha}\cdot  \mathcal{U}_\mathcal{A}^{1-\alpha} \cdot \exp\left(-\lambda Q_{g+1,n+1}\right) \\
			&\propto\pi_{g,n} \exp\left(-\left(\lambda Q_{g+1,n} + (1-\alpha) \log \pi_{g,n}\right)\right) 
			\end{align*}							
		
		The value of $\eta$ assures that $\pi_{g+1,n}$ is normalised, its value is straightforward and it can be absorbed in the proportionality. That of $\lambda$ can be found by substituting $\pi_{g+1,n}$ into  $\mathcal{L}_\pi$ and solving the dual problem. If $\lambda > 0 $, a solution is feasible; therefore, if we choose $\lambda > 0$, then some $\Delta(\pi_g,\alpha,\lambda) > 0$ exists. Further note that $\pi_{g+1,n}$ only takes on positive values in the set $\mathcal{A}$, on account of the uniform distribution $\mathcal{U}_\mathcal{A}$, which, for notational convenience, can also be absorbed into the proportionality. Finally, it can be verified that the corresponding normalisation constant is equal to the definition of $V_{g+1,n}$.
	\end{proof}
	
	\bibliographystyle{unsrt}
	\bibliography{references}

\begin{thebibliography}{10}

\bibitem{mordatch2012trajopt}
I.~Mordatch, E.~Todorov, and Z.~Popovi\'{c}.
\newblock Discovery of complex behaviors through contact-invariant
  optimization.
\newblock {\em ACM Transactions on Graphics}, 31(4), July 2012.

\bibitem{dantec2021icra}
E.~Dantec, R.~Budhiraja, A.~Roig, T.~Lembono, G.~Saurel, O.~Stasse,
  P.~Fernbach, S.~Tonneau, S.~Vijayakumar, S.~Calinon, M.~Taix, and N.~Mansard.
\newblock Whole body model predictive control with a memory of motion:
  Experiments on a torque-controlled talos.
\newblock In {\em 2021 IEEE International Conference on Robotics and Automation
  (ICRA)}, 2021.

\bibitem{todorov2005ilqr}
E.~Todorov and W.~Li.
\newblock A generalized iterative lqg method for locally-optimal feedback
  control of constrained nonlinear stochastic systems.
\newblock In {\em Proceedings of the 2005, American Control Conference, 2005.},
  pages 300--306 vol. 1, 2005.

\bibitem{mayne1966ddp}
D.~Mayne.
\newblock A second-order gradient method for determining optimal trajectories
  of non-linear discrete-time systems.
\newblock {\em International Journal of Control}, 3(1):85--95, 1966.

\bibitem{williams2016agressive}
G.~Williams, P.~Drews, B.~Goldfain, J.~Rehg, and E.~Theodorou.
\newblock Aggressive driving with model predictive path integral control.
\newblock In {\em 2016 IEEE International Conference on Robotics and Automation
  (ICRA)}, pages 1433--1440, May 2016.

\bibitem{williams2017model}
G.~Williams, A.~Aldrich, and E.~Theodorou.
\newblock Model predictive path integral control: From theory to parallel
  computation.
\newblock {\em Journal of Guidance, Control, and Dynamics}, 40(2):344--357,
  2017.

\bibitem{williams2018information}
G.~Williams, P.~Drews, B.~Goldfain, J.~Rehg, and E.~Theodorou.
\newblock Information-theoretic model predictive control: Theory and
  applications to autonomous driving.
\newblock {\em IEEE Transactions on Robotics}, 34(6):1603--1622, 2018.

\bibitem{ha2016path}
J.-S. Ha and H.-L. Choi.
\newblock A topology-guided path integral approach for stochastic optimal
  control.
\newblock In {\em 2016 IEEE International Conference on Robotics and Automation
  (ICRA)}, pages 4605--4612, 2016.

\bibitem{ha2018path}
J.-S. Ha, H.-J. Chae, and H.-L. Choi.
\newblock Approximate inference-based motion planning by learning and
  exploiting low-dimensional latent variable models.
\newblock {\em IEEE Robotics and Automation Letters}, 3(4):3892--3899, 2018.

\bibitem{kahn2021badgr}
G.~Kahn, P.~Abbeel, and S.~Levine.
\newblock Badgr: An autonomous self-supervised learning-based navigation
  system.
\newblock {\em IEEE Robotics and Automation Letters}, 6(2):1312--1319, 2021.

\bibitem{nagabandi2020deep}
A.~Nagabandi, K.~Konolige, S.~Levine, and V.~Kumar.
\newblock Deep dynamics models for learning dexterous manipulation.
\newblock In {\em Conference on Robot Learning}, pages 1101--1112. PMLR, 2020.

\bibitem{bhardwaj2021fast}
Mohak Bhardwaj, Balakumar Sundaralingam, Arsalan Mousavian, Nathan Ratliff,
  Dieter Fox, Fabio Ramos, and Byron Boots.
\newblock Fast joint space model-predictive control for reactive manipulation.
\newblock {\em arXiv preprint arXiv:2104.13542}, 2021.

\bibitem{kappen2005linear}
H.~Kappen.
\newblock Linear theory for control of nonlinear stochastic systems.
\newblock {\em Physical review letters}, 95(20):200201, 2005.

\bibitem{stulp2012path}
F.~Stulp and O.~Sigaud.
\newblock Path integral policy improvement with covariance matrix adaptation.
\newblock In {\em Proceedings of the 29th International Conference on Machine
  Learning (ICML-12)}, 2012.

\bibitem{rajamaki2016sampled}
J.~Rajam{\"a}ki, K.~Naderi, V.~Kyrki, and P.~H{\"a}m{\"a}l{\"a}inen.
\newblock Sampled differential dynamic programming.
\newblock In {\em 2016 IEEE/RSJ International Conference on Intelligent Robots
  and Systems (IROS)}, pages 1402--1409. IEEE, 2016.

\bibitem{ghandi2021ieeeral}
M.~Gandhi, B.~Vlahov, J.~Gibson, G.~Williams, and E.~Theodorou.
\newblock Robust model predictive path integral control: Analysis and
  performance guarantees.
\newblock {\em IEEE Robotics and Automation Letters}, 6(2):1423--1430, 2021.

\bibitem{lefebvre2019path}
T.~Lefebvre and G.~Crevecoeur.
\newblock Path integral policy improvement with differential dynamic
  programming.
\newblock In {\em 2019 IEEE International Conference on Advanced Intelligent
  Mechatronics (AIM)}. IEEE, 2019.

\bibitem{rawlik2013stochastic}
K.~Rawlik, M.~Toussaint, and S.~Vijayakumar.
\newblock On stochastic optimal control and reinforcement learning by
  approximate inference.
\newblock In {\em Twenty-Third International Joint Conference on Artificial
  Intelligence}, 2013.

\bibitem{Watson2021cai}
J.~Watson.
\newblock {\em Control as Inference?}, pages 189--206.
\newblock Springer International Publishing, Cham, 2021.

\bibitem{levine2018reinforcement}
S.~Levine.
\newblock Reinforcement learning and control as probabilistic inference:
  Tutorial and review.
\newblock {\em arXiv preprint arXiv:1805.00909}, 2018.

\bibitem{ziebart2010modeling}
B.~Ziebart, J.~Bagnell, and A.~Dey.
\newblock Modeling interaction via the principle of maximum causal entropy.
\newblock 2010.

\bibitem{toussaint2009robot}
M.~Toussaint.
\newblock Robot trajectory optimization using approximate inference.
\newblock In {\em Proceedings of the 26th annual international conference on
  machine learning}, pages 1049--1056, 2009.

\bibitem{watson2020stochastic}
J.~Watson, H.~Abdulsamad, and J.~Peters.
\newblock Stochastic optimal control as approximate input inference.
\newblock In {\em Conference on Robot Learning}, pages 697--716. PMLR, 2020.

\bibitem{lefebvre2020elsoc}
T.~Lefebvre and G.~Crevecoeur.
\newblock On entropy regularized path integral control for trajectory
  optimization.
\newblock {\em Entropy}, 22(10):1120, 2020.

\bibitem{luo2019minima}
X.~Luo.
\newblock Minima distribution for global optimization.
\newblock {\em arXiv preprint arXiv:1812.03457}, 2019.

\bibitem{oates2019probnum}
C.~Oates and T.~Sullivan.
\newblock A modern retrospective on probabilistic numerics.
\newblock {\em Statistics and computing}, 29(6):1335--1351, 2019.

\bibitem{jaynes2003}
E.~T. Jaynes.
\newblock {\em Probability Theory: The Logic of Science}.
\newblock Cambridge University Press, 2003.

\bibitem{kullback1951information}
S.~Kullback and R.~Leibler.
\newblock On information and sufficiency.
\newblock {\em The annals of mathematical statistics}, 22(1):79--86, 1951.

\bibitem{jaynes1986background}
E.T. Jaynes and James~H. Justice.
\newblock {\em Bayesian Methods: General Background}, page 1–25.
\newblock Cambridge University Press, 1986.

\bibitem{jaynes1982rationale}
E.~Jaynes.
\newblock On the rationale of maximum-entropy methods.
\newblock {\em Proceedings of the IEEE}, 70(9):939--952, 1982.

\bibitem{abdolmaleki2015model}
A.~Abdolmaleki, R.~Lioutikov, J.~Peters, N.~Lau, L.~Reis, and G.~Neumann.
\newblock Model-based relative entropy stochastic search.
\newblock In {\em Advances in Neural Information Processing Systems}, pages
  3537--3545, 2015.

\bibitem{abdolmaleki2017deriving}
A.~Abdolmaleki, B.~Price, N.~Lau, L.~Reis, and G.~Neumann.
\newblock Deriving and improving cma-es with information geometric trust
  regions.
\newblock In {\em Proceedings of the Genetic and Evolutionary Computation
  Conference}, pages 657--664. ACM, 2017.

\end{thebibliography}
	
\end{document}